\documentclass[letterpaper]{article} 
\usepackage{aaai2026}  
\usepackage{times}  
\usepackage{helvet}  
\usepackage{courier}  
\usepackage[hyphens]{url}  
\usepackage{graphicx} 
\usepackage{natbib}  
\usepackage{caption} 
\usepackage{algorithm}
\usepackage{algorithmic}

\usepackage{amssymb}
\usepackage{graphicx}
\usepackage{amsmath}
\usepackage{amsthm}
\usepackage{makecell}

\usepackage{booktabs}

\newtheorem{remark}{Remark}
\newtheorem{theorem}{Theorem}

\usepackage{newfloat}
\usepackage{listings}
\DeclareCaptionStyle{ruled}{labelfont=normalfont,labelsep=colon,strut=off} 
\floatstyle{ruled}
\newfloat{listing}{tb}{lst}{}
\floatname{listing}{Listing}
 \nocopyright 

\title{FS-IQA: Certified Feature Smoothing for Robust Image Quality Assessment}
\author{
    Ekaterina Shumitskaya \textsuperscript{\rm 1,2,3},
    Dmitriy Vatolin  \textsuperscript{\rm 1,2,3}, 
    Anastasia Antsiferova  \textsuperscript{\rm 2,1,4}
}
\affiliations{
    \textsuperscript{\rm 1}
    ISP RAS
    \textsuperscript{\rm 2}
    MSU AI Institute 
    \textsuperscript{\rm 3}
    Lomonosov MSU
    \textsuperscript{\rm 4}
    Innopolis University
}

\usepackage{bibentry}

\begin{document}

\maketitle

\begin{abstract}
We propose a novel certified defense method for Image Quality Assessment (IQA) models based on randomized smoothing with noise applied in the feature space rather than the input space. Unlike prior approaches that inject Gaussian noise directly into input images, often degrading visual quality, our method preserves image fidelity while providing robustness guarantees. To formally connect noise levels in the feature space with corresponding input-space perturbations, we analyze the maximum singular value of the backbone network’s Jacobian. Our approach supports both full-reference (FR) and no-reference (NR) IQA models without requiring any architectural modifications, suitable for various scenarios. It is also computationally efficient, requiring a single backbone forward pass per image. Compared to previous methods, it reduces inference time by \textbf{99.5\%} without certification and by \textbf{20.6\%} when certification is applied. We validate our method with extensive experiments on two benchmark datasets, involving six widely-used FR and NR IQA models and comparisons against five state-of-the-art certified defenses. Our results demonstrate consistent improvements in correlation with subjective quality scores by up to \textbf{30.9\%}. Code is publicly available at \textit{link is hidden for a blind review}.
\end{abstract}


\section{Introduction}

\begin{figure}[htb]
\begin{center}
\centerline{\includegraphics[width=\linewidth]{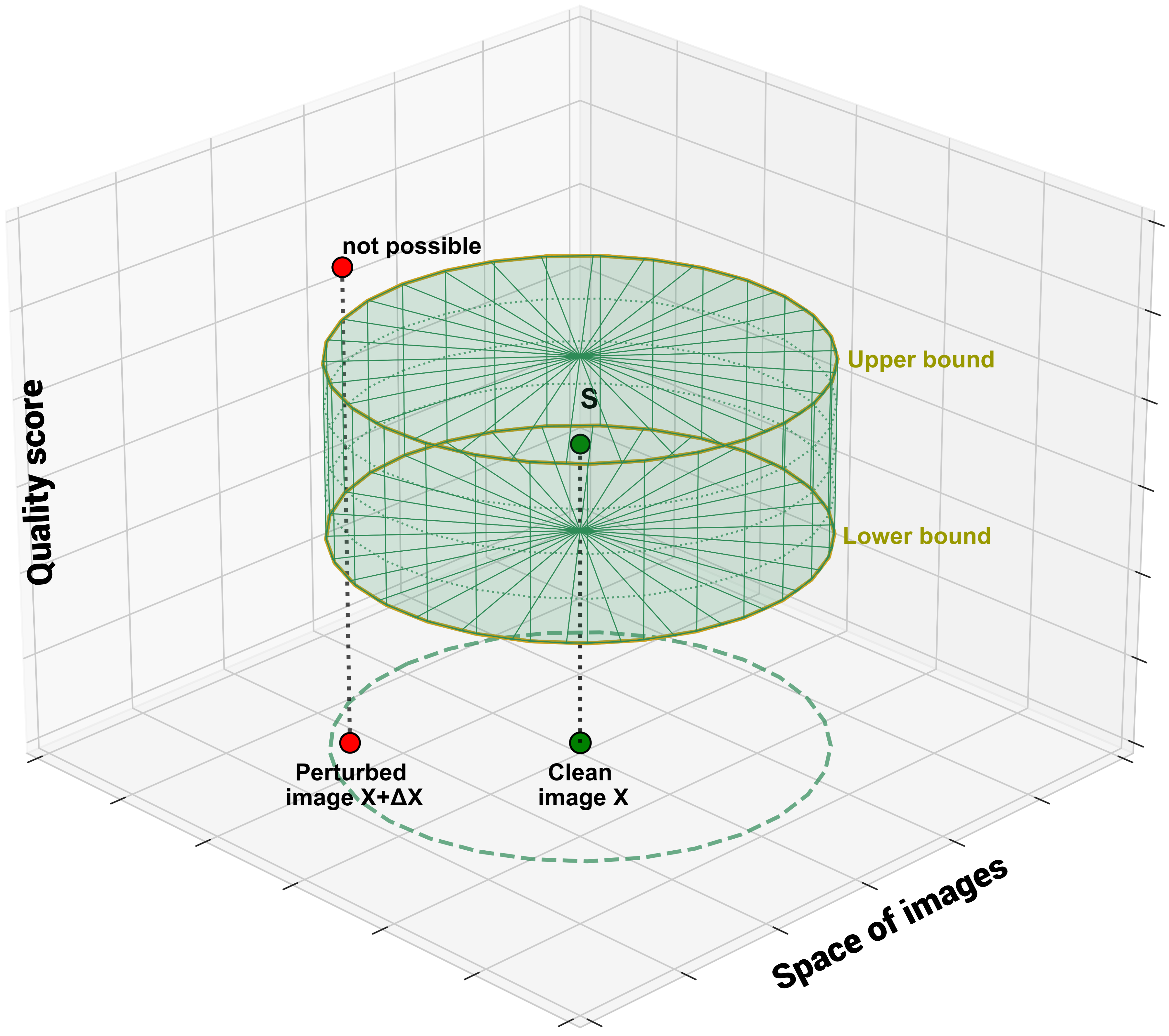}}
\caption{Visualization of a multidimensional cylinder bounding predicted quality score variations under input-space adversarial perturbations. The volume of the cylinder scales with feature-space noise level.}
\label{fig:intro-image}
\end{center}
\end{figure}

Image Quality Assessment (IQA) plays a crucial role in numerous applications, from image processing algorithms development to medical imaging and video streaming. Accurate and robust IQA models are essential to reliably measure the perceptual image quality under varying conditions. However, recent studies have shown that IQA models are vulnerable to adversarial perturbations \cite{zhang2022perceptual, yang2024exploring, antsiferova2024comparing}, which can lead to inaccurate quality scores and compromise their trustworthiness \cite{deng2024sparse, yu2025backdoor, visapp25}. 

To address this challenge, certified defenses for IQA have emerged as a promising direction, providing robustness guarantees by either constraining the model architecture \cite{ghazanfari2023lipsim} or injecting Gaussian noise at the input image level \cite{SHUMITSKAYA2025104447}. Still, these solutions aren’t perfect. Architectural constraints often degrade model accuracy, reducing the correlation between predicted quality scores and true subjective scores. Meanwhile, input-space noise augmentation like randomized smoothing or median smoothing \cite{cohen2019certified, chiang2020detection}, though theoretically effective, distorts the original image content. These distortions suppress or affect important visual feature — like fine textures, edges, or compression artifacts — which play a fundamental role in assessing perceptual quality. For instance, if subtle compression artifacts that degrade image quality are smoothed out or masked by added noise, the model may fail to detect these degradations properly, resulting in lower correlation between predicted quality scores and actual human subjective assessments. This trade-off between robustness and accuracy limits the practical usage of such defenses.

To overcome this limitation, we propose to shift the smoothing operation from the input space to a more semantically meaningful feature space. Unlike input-space noise, smoothing features instead of raw images can help preserve critical quality-related information while still providing robustness guarantees. Feature smoothing has been explored in prior work, primarily for developing defenses in classification tasks \cite{addepalli2021boosting, ma2023adversarial}. However, these approaches are empirical and do not establish a formal relationship between noise levels in the feature space and those in the input space. Our research aims to fill this gap by proposing a certified defense applied in the feature space specifically for the IQA task.

The main purpose of our approach is to identify a theoretical multidimensional cylinder that tightly bounds the potential variations of predicted quality scores under adversarial perturbations in the image space (see Figure \ref{fig:intro-image} for visualization). By adjusting the noise level in the feature space, we control the volume of this cylinder — the higher the noise level, the larger the volume. 

Our method decomposes the IQA model into two components: a backbone and a scorer module. This design eliminates the need to retrain the backbone network; only the scorer requires fine-tuning. The backbone processes the input image to produce a feature representation, where randomized smoothing is applied. Subsequently, the scorer module generates the final quality score along with robustness certificates. To formally relate the feature space noise to the input perturbation magnitude, we analyze the maximum singular value of the backbone’s Jacobian matrix. More precisely, for a given image, our method outputs a tuple \((S, R, S^l, S^u)\), where \(S\) denotes the predicted quality score, and \(R\) represents the input-space radius within which the output of the defended IQA model is guaranteed to lie between \(S^l\) (lower bound) and \(S^u\) (upper bound). This corresponds to a multidimensional cylinder that bounds quality score variations under allowable input perturbations.

The primary contributions of this work can be summarized as follows:

\begin{itemize}
    \item To the best of our knowledge, this work presents the first certified defense for IQA, that operates in the feature space instead of the input space, preserving image quality while providing certified robustness. Our method supports both full-reference (FR) and no-reference (NR) IQA models without requiring any architectural modifications.
    \item We theoretically connect feature-space noise levels to input-space perturbations by analyzing the maximum singular value of the backbone’s Jacobian. 
    \item We conduct extensive experiments on \textbf{two datasets} using \textbf{six} widespread FR and NR \textbf{IQA models}, comparing our approach against \textbf{five} existing state-of-the-art methods. Our method consistently improves the correlation with subjective quality scores by up to \textbf{30.9\%}.
    \item Our method is computationally efficient, requiring only one backbone forward pass per image. Compared to existing techniques, it reduces inference time by \textbf{99.5\%} without certification and by \textbf{20.6\%} with certification.
    \item We analyze the method's usability beyond certified guarantees and demonstrate that it suppresses adversarial gain, improving empirical robustness by \textbf{69.9\%}.
\end{itemize}

Our code is publicly available at \textit{link is hidden for a blind review}.

\begin{figure*}[htb]
\begin{center}
\centerline{\includegraphics[width=\linewidth]{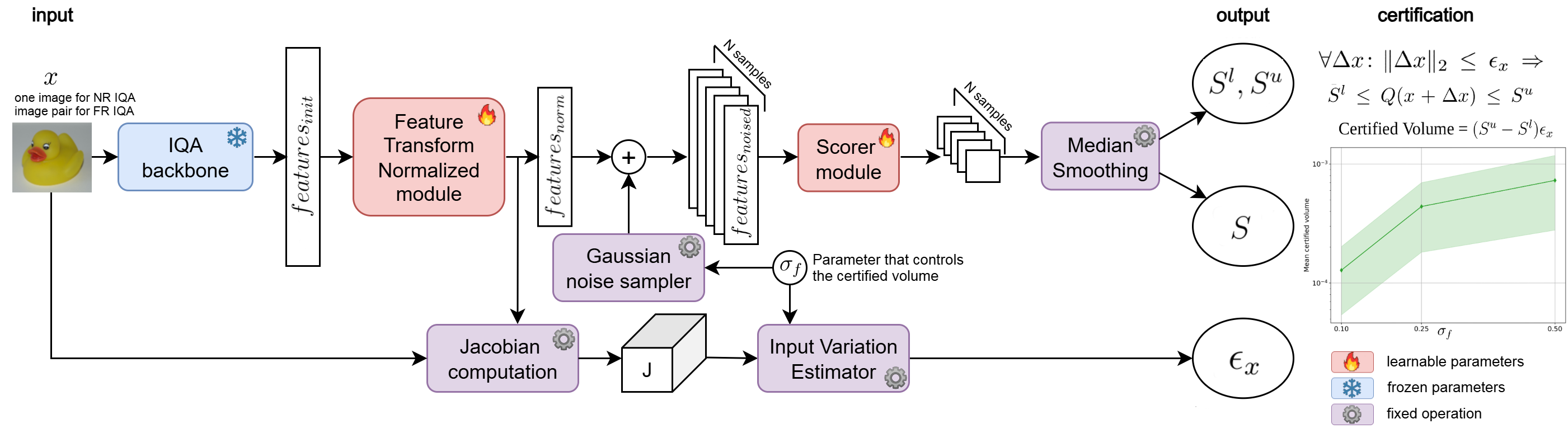}}
\caption{Overview of the proposed FS-IQA defense. The input image or image pairs are first processed by the IQA backbone and the FTN module. Gaussian noise is then added in the feature space. Using Median Smoothing theory \cite{chiang2020detection}, we derive the final quality score $S$ along with its lower and upper certified bounds $(S^l, S^u)$. Finally, the proposed Input Variation estimator is used to determine the corresponding input perturbation constraints $\epsilon_x$ for the given feature noise level $\sigma_f$.}
\label{fig:method_overview}
\end{center}
\end{figure*}

\section{Related Work}

\subsection{Preliminaries. Randomized Smoothing}
Randomized Smoothing \textbf{\textit{(RS)}} \cite{cohen2019certified} is a widely used certification technique that provides provable robustness guarantees for large-scale models, requiring only black-box access to model evaluations. Originally introduced for classification tasks, RS constructs a smoothed classifier $g$ from a base classifier $f$ by averaging predictions over Gaussian noise perturbations $e$ added to the input $x$:  

\[
    g(x) = \mathbb{E} \{ f(x + e) \}, e \sim \mathcal{N}(0, \sigma^2I).
\]  

This approach relies on the robustness of $f$ under Gaussian noise to certify that $g$ is resistant to adversarial perturbations within an $l_2$-norm ball of radius $\epsilon$: 
\[
\epsilon = \frac{\sigma}{2} \bigl(\Phi^{-1}(pA) - \Phi^{-1}(pB)\bigr),
\]  
where \( pA \) and \( pB \) are lower and upper confidence bounds on the top and second-top class probabilities, respectively, and \( \Phi^{-1} \) denotes the inverse Gaussian cumulative distribution function. 

\subsection{Randomized Smoothing for regression}

In this section, we review existing approaches for defending IQA models as well as regression models in general. Since regression models produce continuous outputs for given inputs, developing certified defenses for them is fundamentally more difficult compared to discrete tasks such as classification.

\textbf{\textit{RS-Reg \cite{rekavandi2025rs}.}} In this paper, the authors generalize randomized smoothing for application to regression models. They define a smoothed function as $g(x) = mean \{f(x+e) \}, e \sim \mathcal{N}(0, \sigma^2I) $, where $x$ is an input image, $f$ is the base regression model and $e$ is Gaussian noise sampled from a multivariate normal distribution with covariance matrix $\sigma^2 I$. Using this definition, the authors further derive a probabilistic certified upper bound on input perturbations for the base regression model, assuming the outputs are bounded.  

\textbf{\textit{Cert-Reg \cite{rekavandi2024certified}.}} The authors of this paper extend randomized smoothing to regression models using powerful tools from robust statistics, specifically using $\alpha$-trimming filter as the smoothing function: $g(x) = \frac{1}{N - 2 [\alpha N]} \sum_{i = [\alpha N]+1}^{N-[\alpha N]} f(x+e), e \sim \mathcal{N}(0, \sigma^2I)$. Here, the notation is consistent with the previous paper, the values $f(x+e)$ are assumed to be sorted, $N$ denotes the total number of noise samples, and $\alpha \in [0,0.5)$ is the trimming parameter. In other words, this operator removes the lowest $\alpha$-fraction and the highest $\alpha$-fraction of outputs and computes the average of the remaining ones. This approach improves robustness by reducing sensitivity to outliers. Moreover, the authors derive a certification radius for the input perturbation under the assumption that the model outputs are bounded.

\textbf{\textit{MS, DMS \cite{chiang2020detection}}.} In this paper, the authors propose a novel defense approach for regression models using the median as the smoothing function:  
$g(x) = \mathrm{median} \{ f(x + e) \}, \quad e \sim \mathcal{N}(0, \sigma^2 I).$  
This operator is significantly more robust to outliers compared to both averaging and the $\alpha$-trimming filter. To provide certification, the authors prove theorems imposing restrictions on the model’s output within a fixed \(\ell2\)-norm ball around the input. Furthermore, they address the potential loss in model performance caused by the added noise and propose an extension of their method that includes an additional denoising step applied after adding noise but before evaluating the regression model's output. In this paper, we refer to this enhanced method as Denoised Median Smoothing (DMS), while the original approach is called Median Smoothing (MS).

\textbf{\textit{DMS-IQA \cite{SHUMITSKAYA2025104447}}.} In this paper, the authors extend the concept of Median Smoothing specifically to defend IQA models. They propose a novel denoiser training scheme based on a composite loss function consisting of three components: pixel-wise MSE, MSE between predicted and true subjective scores, and a differentiable rank loss. Their experimental results demonstrate that this composite loss function enables the training of better denoiser for use within the Median Smoothing framework.

\subsection{Defenses for IQA}

Various empirical defense methods for IQA models have been proposed, including image purification \cite{gushchin2024guardians, liu2025enhancing} and adversarial training \cite{liu2024defense, chistyakova2024increasing}. However, research on certified defenses specifically designed for IQA remains in its early stages, with only a few existing works addressing this challenge. For example, LipSim \cite{ghazanfari2023lipsim} proposes using Lipschitz networks to certify a FR IQA model. While this provides robustness guarantees, it is not a defense mechanism in itself, but rather a single robust FR IQA architecture. As a result, this approach lacks scalability to NR IQA models and other FR IQA architectures. Another relevant method, DMS-IQA \cite{SHUMITSKAYA2025104447}, applies randomized smoothing in the image space as a defense for IQA models. However, this approach perturbs the images themselves, which degrades image quality and affect the accuracy of quality assessment.

To address these challenges, we introduce a certified defense based on feature-space smoothing, which preserves image fidelity while delivering robust and scalable defense applicable to both FR and NR IQA models.


\section{Proposed Method}

\subsection{Notation and Problem Statement}

First, we introduce the notation. We define the defended model \( Q(\cdot) \), which maps an input image (or image pairs) to a tuple of outputs:  
\[
Q : x \in \mathbb{R}^{3 \times H \times W} \to (S, \epsilon_x,  S^l, S^u),
\]
where \( x \) is an input image for NR IQA and an image pair for FR IQA, \( S \) is the predicted quality score of \( x \), \( S^l \) and \( S^u \) the lower and upper bounds on the quality score under allowable perturbations, and \( \epsilon_x \) the maximum perturbation magnitude of \( x \). 

Our objective is to guarantee that for any input image perturbation \( \Delta x \) satisfying  $||\Delta x||_2 \leq \epsilon_x$, the predicted quality score remains within the bounds:  
\[
S^l_{Q(x)} \leq S_{Q(x + \Delta x)} \leq S^u_{Q(x)}.
\]

Note that for FR IQA, the perturbation \( \Delta x \) refers to perturbation of the distorted image only, while the reference image is considered fixed.

\subsection{Overview of the FS-IQA Architecture}

Figure \ref{fig:method_overview} provides an overview of the proposed FS-IQA method. The defended model \( Q \) is composed of three CNN-based modules: the IQA backbone \( b(\cdot) \), the Feature Transform Normalize Module \( FTN(\cdot) \), and the scorer module \( S(\cdot) \). Beyond the neural modules, FS-IQA includes two components that provide certified robustness guarantees: the Median Smoothing operator \cite{chiang2020detection} and the proposed Input Variation Estimator. These components will be described in detail in the following subsections.

\subsection{Feature Preparation Pipeline}

First, the input image \( x \) is processed by the IQA backbone \( b(\cdot) \), producing the initial features:  

\[
f_{init} = b(x)
\]

Next, these features are passed through the Feature Transform Normalize Module to obtain normalized features:  

\[
f_{norm} = FTN(f_{init})
\]

The primary objectives of the Feature Transform Normalize Module are the following: 

\begin{itemize}
    \item Dimensionality reduction — the module reduces the feature dimensionality to 512. This step is important for memory-efficient computation of the Jacobian matrix of the module’s output with respect to the input image. 
    \item Normalization — the module constrains the feature values within the \([0, 1]\) range. This normalization prepares the features for the subsequent consistent application of Gaussian noise.
\end{itemize}

Then, we generate \( N \) samples of Gaussian noise and add them to the normalized features to obtain smoothed features:  
\[
\{ f_{noised}^i = f_{norm} + e_i \mid e_i \sim \mathcal{N}(0, \sigma_f^2 I), \quad i = 1, \ldots, N \},
\]  
where \(\sigma_f\) denotes the noise standard deviation, \(I\) is the identity matrix, and \(N\) is the number of samples.

\subsection{Input Variation Estimator}
Next, given the feature noise standard deviation \(\sigma_f\), we use the proposed Input Variation Estimator module to determine the maximum allowable perturbation \(\epsilon_x\) on the input image in terms of its \(l_2\) norm, i.e.,  
\[
\epsilon_x = IVE(FTN \circ b, x, \sigma_f).
\]  
The value \(\epsilon_x\) quantifies how much the input image \(x\) can be perturbed without causing feature variations larger than \(\sigma_f\). To estimate \(\epsilon_x\), we rely on a linear approximation of the composite function $B = FTN \circ b$ around \(x\), formalized in the following theorem:

\begin{theorem}[Input Variation Estimator]
Let \( B: \mathbb{R}^n \to \mathbb{R}^m \) be a differentiable feature map at point \( x \), and let \( J_B(x) \in \mathbb{R}^{m \times n} \) be its Jacobian at \( x \). For any input perturbation \( u \in \mathbb{R}^n \) satisfying
\[
\|u\|_2 \leq \epsilon_x,
\]
the change in features satisfies
\[
\max_{\|u\|_2 \leq \epsilon_x} \|B(x + u) - B(x)\|_2 \leq \sigma_f.
\]
Using the first-order approximation \( B(x+u) \approx B(x) + J_B(x) u \), the maximal allowed input perturbation \( \epsilon_x \) is given by
\[
\epsilon_x = \frac{\sigma_f}{\|J_B(x)\|_2},
\]
where \(\|J_B(x)\|_2\) denotes the spectral norm (operator norm induced by the Euclidean norm) of the Jacobian matrix.
\end{theorem}

\begin{proof}
By definition, the maximal input perturbation \(\epsilon_x\) is the largest radius \(\rho\) such that for every vector \(u\) with \(\|u\|_2 \leq \rho\), the corresponding feature change does not exceed \(\sigma_f\):
\[
\epsilon_x = \underset{\rho \geq 0}{\max}\left\{ \forall u, \|u\|_2 \leq \rho \implies \|B(x + u) - B(x)\|_2 \leq \sigma_f \right\}.
\]

For sufficiently small perturbations \(u\), we can use the first-order Taylor expansion: $B(x+u) \approx B(x) + J_B(x) u,$
where \(J_B(x)\) is the Jacobian matrix at the point \(x\). This approximation implies that the change in features is approximately linear with respect to \(u\), so
\[
\|B(x+u) - B(x)\|_2 \approx \|J_B(x) u\|_2.
\]

\[
\|J_B(x) u\|_2 \leq \|J_B(x)\|_2 \|u\|_2.
\]

To ensure the feature change never exceeds \(\sigma_f\) for all \(u\) with \(\|u\|_2 \leq \rho\), it suffices to require
\[
\|J_B(x)\|_2 \rho \leq \sigma_f,
\]
which rearranges to
\[
\rho \leq \frac{\sigma_f}{\|J_B(x)\|_2}.
\]

Since \(\epsilon_x\) is defined as the maximal such \(\rho\), we get:
\[
\epsilon_x = \frac{\sigma_f}{\|J_B(x)\|_2}.
\]

\end{proof}

\begin{remark}
The spectral norm \(\|J_B(x)\|_2\) of the Jacobian matrix is equal to its largest singular value \(\sigma_{\max}(J_B(x))\). This equivalence is important: the largest singular value represents the greatest factor by which the feature map's linear approximation can amplify the input perturbation in any direction. Bounding the input perturbation via this norm therefore provides control over the worst-case behavior of the feature change.
\end{remark}

\subsection{Median Smoothing in the feature space}
To obtain the final quality score, as well as the lower and upper bounds on the model output, we use Median Smoothing theory \cite{chiang2020detection}. Final quality score calculates as the value of Median Smoothing operator, applied to the set of smoothed features:

\[
\overline{v} = \{ \mathrm{Scorer}(f_{\mathrm{noised}}^{i}) \mid i = 1, \ldots, N \}
\]

\begin{equation}
S = \mathrm{median}\left( \overline{v} \right)
\label{eq:med_sm}
\end{equation}

\[
(S^l, S^u) = MS_{\mathrm{Cert}}\left(\overline{v}, \sigma_f, N, \alpha \right)
\]

where \(\alpha\) is the confidence level, \(MS_{\mathrm{Cert}}(\cdot)\) is an operator that provides certified guarantees for Median Smoothing when \(\mathrm{Scorer}(\cdot)\) is a regression function, i.e., a mapping from \(\mathbb{R}^k\) to \(\mathbb{R}\). Formally, these guarantees follow from the theorem below:

\begin{theorem}
\cite{chiang2020detection} If $Scorer(\cdot)$ is a regression function from the space of features to the space of real numbers, $f_{norm}$ is normalized features, $G$ is the Median Smoothing operator in the form (\ref{eq:med_sm}), then for all $\|u\|_2 \le \varepsilon_f$:
\begin{equation}
S^l \leq G(S(f_{norm}+u)) \leq S^u.
\label{eq:theorem1}
\end{equation}
Here, $\underline{p} = \Phi(- \frac{\varepsilon_f}{\sigma_f})$, $\overline{p} = \Phi(\frac{\varepsilon_f}{\sigma_f}) $ and $\Phi(\cdot)$ is the Gaussian cumulative density function. $G$ is an operator of Median Smoothing, $S^l$ and $S^u$ are defined as the percentiles of the smoothed function:
\begin{align}
    &S^l = \text{sup}_{y \in \mathbb{R}}\{\mathbb{P}_{r \sim \mathcal{N}(0, \sigma_f^2 I)}[S(f_{norm}+r) \leq y] \leq \underline{p} \}, \nonumber \\ 
    &S^u = \text{inf}_{y \in \mathbb{R}} \{\mathbb{P}_{r \sim \mathcal{N}(0, \sigma_f^2 I)} [S(f_{norm}+r) \leq y] \geq \overline{p} \}.  \nonumber 
\end{align}
\end{theorem}

\subsection{Training and application}

The parameters of the IQA backbone \( b(\cdot) \) are frozen, while \( FTN(\cdot) \) and \( Scorer(\cdot) \) are trainable. The training objective focuses solely on predicting the quality scores \( S \), since the primary goal of the IQA task is to produce scores that highly correlate with human assessments. Accordingly, we use the mean squared error (MSE) loss between the predicted scores $S$ and the ground-truth subjective quality scores.

Pseudocode for the inference procedure of the FS-IQA model is provided in Algorithm \ref{alg:quality_certification}.

\begin{algorithm}[H]
\caption{\textbf{Pseudocode} for Quality Prediction and Certification for Feature-Smoothed IQA Model $Q$ on $x$}
\label{alg:quality_certification}
\textbf{Input}: Image (or image pair) \(x\), feature noise bound \(\sigma_f\), number of samples \(N\), confidence level \(\alpha\), Jacobian norm threshold \(\tau > 0\) \\
\textbf{Output}: Quality score \(S\), certified input bound \(\epsilon_x\), and lower and upper output bounds \(S^l, S^u\) such that \(\forall \Delta x \colon \|\Delta x\|_2 \le \epsilon_x \Rightarrow S^l \le Q(x+\Delta x) \le S^u\), or \texttt{ABSTAIN} if certification is not reliable
\begin{algorithmic}[1]
\STATE \(f_{\text{init}} \gets IQA_{\text{backbone}}(x)\) \COMMENT{Extract initial features}
\STATE \(f_{\text{norm}} \gets FTN(f_{\text{init}})\) \COMMENT{Apply Feature Transform Normalization}
\STATE \(J \gets \text{Jacobian}\big(FTN \circ IQA_{\text{backbone}}, x\big)\) \COMMENT{Compute Jacobian at \(x\)}
\IF{\(\|J\|_2 < \tau\)} 
    \RETURN \texttt{ABSTAIN} \COMMENT{Jacobian norm too small — certification not reliable}
\ENDIF
\STATE \(\epsilon_x \gets \frac{\sigma_f}{\|J\|_2}\) \COMMENT{Maximal allowed input perturbation}

\STATE Initialize empty list \texttt{scores}
\FOR{\(i = 1 \to N\)}
    \STATE Sample noise vector \(e \sim \mathcal{N}(0, \sigma_f I)\)
    \STATE \(f_{\text{noised}} \gets f_{\text{norm}} + e\) \COMMENT{Add noise in feature space}
    \STATE \(S_{cur} \gets Scorer(f_{\text{noised}})\)
    \STATE Append \(S_{cur}\) to \texttt{scores}
\ENDFOR
\STATE \(S \gets \text{median}(\texttt{scores})\)
\STATE \((S^l, S^u) \gets MS_{\text{Cert}}(\texttt{scores}, \sigma_f, N, \alpha)\)
\RETURN \(S, S^l, S^u, \epsilon_x\)
\end{algorithmic}
\end{algorithm}

Experimental verification of theoretical constraints can be found in the supplementary material.

\section{Experiments}

\subsection{FS-IQA parameters}
The $FTN$ and $Scorer$ modules have a simple architecture consisting of fully connected layers. More details in the supplementary material. In our experiments, we use a sample size of \( N = 2000 \), a confidence level of \( \alpha = 0.999 \), and a Jacobian norm threshold \( \tau = 0.001 \). To benchmark the effectiveness of our FS-IQA approach, we conducted comparisons with five certified defenses on NR and FR IQA models. Experiments were conducted using a GPU server powered by NVIDIA A100 GPUs.

\subsection{NR and FR IQA models}
As NR models we selected \textbf{DBCNN} \cite{zhang2020blind}, \textbf{HyperIQA} \cite{Su_2020_CVPR} and \textbf{KonCept} \cite{hosu2020koniq}. As FR models we selected \textbf{LPIPS} \cite{zhang2018perceptual}, \textbf{DISTS} \cite{ding2020iqa} and \textbf{DreamSim} \cite{fu2023dreamsim}. These models are widely recognized in the field of IQA due to their strong performance and their diverse architectural approaches \cite{NEURIPS2022_59ac9f01}.



\subsection{Compared Methods}
For comparison, we include all known certified defense methods to date, both those created specifically for IQA and those made for general regression tasks that can be adapted to IQA. These methods are: \textbf{RS-Reg} \cite{rekavandi2025rs} and \textbf{Cert-Reg} \cite{rekavandi2024certified}, proposed for classification; \textbf{MS} and \textbf{DMS} \cite{chiang2020detection}, proposed for detection; and \textbf{DMS-IQA} \cite{SHUMITSKAYA2025104447}, designed specifically for IQA. Detailed descriptions of these methods are provided in the Related Work section.

\begin{table*}[h!]
\begin{center}
\begin{tabular}{cccccc}
\hline
& \multicolumn{3}{c}{\textbf{SRCC/PLCC}} & \multicolumn{2}{c}{\textbf{time}} \\
\textbf{Method}  & \makecell{$\sigma=0.1$} & $\sigma=0.25$ & $\sigma=0.5$ & with cert & no cert \\
\hline
RS-Reg & 0.66$\pm$0.02/0.67$\pm$0.04 & 0.38$\pm$0.07/0.38$\pm$0.06 & -0.04$\pm$0.05/-0.02$\pm$0.06 & \underline{4.2$\pm$1.7 sec} & \underline{3.8$\pm$1.7 sec} \\
Cert-Reg & 0.66$\pm$0.02/0.67$\pm$0.04 & 0.38$\pm$0.07/0.38$\pm$0.06 & -0.04$\pm$0.05/-0.02$\pm$0.06 & \underline{4.2$\pm$1.7 sec} & 4.2$\pm$1.7 sec \\
MS & 0.66$\pm$0.02/0.67$\pm$0.04 & 0.38$\pm$0.07/0.38$\pm$0.06 & -0.04$\pm$0.05/-0.02$\pm$0.06 & \underline{4.2$\pm$1.7 sec} & \underline{3.8$\pm$1.7 sec} \\
DMS & 0.84$\pm$0.04/0.87$\pm$0.03 & 0.74$\pm$0.05/0.78$\pm$0.05 & 0.63$\pm$0.04/0.67$\pm$0.05 & 7.9$\pm$1.8 sec & 7.4$\pm$1.7 sec \\
DMS-IQA & \underline{0.86$\pm$0.02/0.88$\pm$0.02} & \underline{0.80$\pm$0.03/0.83$\pm$0.02} & \underline{0.70$\pm$0.03/0.73$\pm$0.02} & 7.9$\pm$1.8 sec & 7.4$\pm$1.7 sec \\
FS-IQA (ours) & \textbf{0.91$\pm$0.01/0.93$\pm$0.01} & \textbf{0.91$\pm$0.01/0.93$\pm$0.01} & \textbf{0.91$\pm$0.02/0.92$\pm$0.01} & \textbf{6.9$\pm$2.3 sec} & \textbf{33$\pm$8 ms} \\
\hline
\end{tabular}
\caption{Comparison of methods on NR IQA models (DBCNN, HyperIQA and KonCept) using the KonIQ-10k dataset. All metrics are averaged across the IQA models. Data are presented as mean $\pm$ standard error of the mean (SEM). Time shows time to process one image with resolution $384\times512$. Detailed results for each IQA model are in the supplementary material. }
\label{tab:comparison_nr}
\end{center}
\end{table*}

\begin{table*}[h!]
\begin{center}
\begin{tabular}{cccccc}
\hline
& \multicolumn{3}{c}{\textbf{SRCC/PLCC}} & \multicolumn{2}{c}{\textbf{time}} \\
 \textbf{Method} & \makecell{$\sigma=0.1$} & $\sigma=0.25$ & $\sigma=0.5$ & with cert & no cert \\
\hline
RS-Reg & 0.54$\pm$0.08/0.54$\pm$0.05 & 0.33$\pm$0.12/0.33$\pm$0.13 & 0.14$\pm$0.11/0.13$\pm$0.13 & \textbf{6.9$\pm$2.9} sec & \underline{6.4$\pm$2.9 sec} \\
Cert-Reg & 0.54$\pm$0.08/0.54$\pm$0.05 & 0.33$\pm$0.12/0.33$\pm$0.13 & 0.14$\pm$0.11/0.13$\pm$0.13 & \textbf{6.9$\pm$2.9} sec & 6.9$\pm$2.9 sec \\
MS & 0.54$\pm$0.08/0.54$\pm$0.05 & 0.33$\pm$0.12/0.33$\pm$0.13 & 0.14$\pm$0.11/0.13$\pm$0.13 & \textbf{6.9$\pm$2.9} sec & \underline{6.4$\pm$2.9 sec} \\
DMS & 0.75$\pm$0.02/0.70$\pm$0.04 & 0.62$\pm$0.01/0.62$\pm$0.02 & \underline{0.50$\pm$0.01/0.53$\pm$0.01} & 10.5$\pm$2.9 sec & 10.0$\pm$2.9 sec \\
DMS-IQA & \underline{0.78$\pm$0.02/0.72$\pm$0.04} & \underline{0.63$\pm$0.01/0.62$\pm$0.02} & 0.49$\pm$0.04/0.51$\pm$0.03 & 10.5$\pm$2.9 sec & 10.0$\pm$2.9 sec \\
FS-IQA (ours) & \textbf{0.89$\pm$0.01/0.89$\pm$0.02} & \textbf{0.90$\pm$0.01/0.89$\pm$0.02} & \textbf{0.89$\pm$0.01/0.89$\pm$0.01} & \underline{7.5$\pm$3.6 sec} & \textbf{40$\pm$14 ms} \\
\hline
\end{tabular}
\caption{Comparison of methods on FR IQA models (LPIPS, DISTS and DreamSim) using the Kadid-10k dataset. All metrics are averaged across the IQA models. Data are presented as mean $\pm$ standard error of the mean (SEM). Time shows time to process one image with resolution $384\times512$. Detailed results for each IQA model are in the supplementary material. }
\label{tab:comparison_fr}
\end{center}
\end{table*}

\subsection{Datasets}
For experiments on NR IQA models, we use the \textbf{KonIQ-10k} image dataset \cite{hosu2020koniq}, consisting of 10,073 images. For FR models, we use the \textbf{Kadid-10k} database \cite{lin2019kadid}, consisting of 10,125 images. These datasets are commonly used for designing IQA models. All images have a resolution of $384 \times 512$ pixels. Both datasets were split into training and testing subsets, using an 80\% / 20\% ratio. For defense methods that require a training stage (DMS, DMS-IQA, and the proposed FS-IQA), the training subset was used. All methods were then evaluated on the testing subset using a sample size of \( N=2000 \) and a confidence level of \( \alpha = 0.999 \). Other parameters are detailed in the supplementary material.

\subsection{Evaluation metrics}

As evaluation metrics for the compared methods, we primarily use Spearman’s rank correlation coefficient (\textbf{SRCC}) and Pearson’s linear correlation coefficient (\textbf{PLCC}) with human-assessed subjective scores, since the main goal of IQA is to accurately estimate image quality. Additionally, for comparison, we measure the computational \textbf{time} required to process one image, both with and without certification. The runtime was averaged over 100 runs.

\begin{figure*}[htb]
\begin{center}
\centerline{\includegraphics[width=0.88\linewidth]{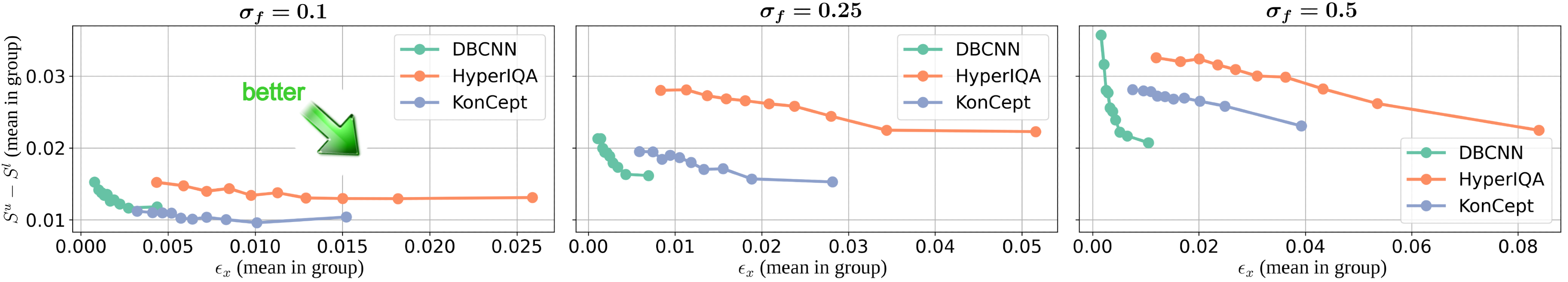}}
\caption{Average difference between upper and lower bounds $(S^u - S^l)$ versus input perturbation \(l_2\) norm constraint \(\epsilon_x\) for three NR IQA models (DBCNN, HyperIQA, KonCept) at noise scales \(\sigma \in \{0.1, 0.25, 0.5\}\), produced using FS-IQA. Data are grouped into 10 quantile-based bins by \(\epsilon_x\), with mean values plotted for each group. }
\label{fig:FScert}
\end{center}
\end{figure*}

\section{Results}

Tables \ref{tab:comparison_nr} and \ref{tab:comparison_fr} present the comparison results of FS-IQA with prior certified defenses on NR IQA and FR IQA models, respectively. For RS-Reg, Cert-Reg, and MS methods, no significant differences were observed. This suggests that mean, \(\alpha\)-trimmed mean, and median data reduction techniques yield similar performance for IQA models. Since IQA models are typically trained on noisy data and are robust to noise outliers, developing complex methods specifically for outlier handling is unnecessary. However, these three methods exhibit poor SRCC and PLCC scores, which dramatically decrease as \(\sigma\) increases. DMS and DMS-IQA methods show significantly better SRCC and PLCC results, though at the cost of increased computation time due to the denoising step. Finally, the proposed FS-IQA method demonstrates the best performance in terms of SRCC and PLCC scores, with only minor decreases as \(\sigma\) rises. This indicates that FS-IQA leverages more semantic features rather than relying solely on input images, making the defense pipeline adaptable and effective even under high noise levels. On average, FS-IQA improves SRCC by approximately \textbf{31.3\%} and PLCC by \textbf{30.5\%} compared to the best existing DMS-IQA method across the three \(\sigma_f\) values.

Regarding computational time, FS-IQA with certification mode is slightly slower than classic RS-Reg, Cert-Reg, and MS methods, but approximately \textbf{20.6\%} faster than the state-of-the-art DMS and DMS-IQA approaches. We also evaluated the methods in a non-certification mode — i.e., when only the certified quality score is computed without calculating restrictions. This mode is useful for real-time applications that use a reliable IQA model proven robust on specific data types, without spending extra time checking restrictions for each item. As observed, all prior certified methods struggle to deliver fast performance in this setting because they require running the IQA backbone multiple times. In contrast, our approach shifts the smoothing operation to the feature space, overcoming this limitation. This results in a dramatic speed advantage: 33 milliseconds versus the second-best time of 7.4 seconds. The difference is significant, highlighting FS-IQA’s practical benefits for real-time applications.

Figure \ref{fig:FScert} presents a visualization of the certified guarantees provided by FS-IQA for NR IQA models at different noise levels \(\sigma_f\). The graph shows the average difference between the upper and lower bounds versus \(\epsilon_x\), grouping the data into 10 quantile-based bins according to \(\epsilon_x\), with mean values plotted for each group.  As observed, increasing \(\sigma_f\) results in wider bounds, meaning that by adjusting the \(\sigma_f\) level, we can control the strength of the certification guarantees. Additionally, as shown in Figure \ref{fig:FScert}, for IQA models it is preferable to be located in the bottom-right corner of the plot, which indicates that the model’s output changes very little within a large neighborhood around the input — thus providing stronger guarantees. This visualization allows us to compare different IQA backbones and select the best option. For example, the KonCept line consistently lies below the HyperIQA line across all \(\sigma_f\) levels, indicating that for the same \(\epsilon_x\) values, KonCept imposes tighter restrictions on the model output. Therefore, we can conclude that KonCept, combined with FS-IQA, provides stronger certified guarantees than HyperIQA.


\section{Discussion}

\subsection{Beyond Certified Guarantees}
It is a popular topic of discussion that today's certified defenses are not practically useful because they provide tight certified guarantees at the cost of increased computation time. Additionally, questions arise about how these methods can defend against much larger perturbations. To explore this issue, we conducted additional experiments to evaluate the empirical robustness of our approach against perturbations with amplitudes much more greater than the certified robustness guarantee. Specifically, we applied I-FGSM (10 iterations) \cite{kurakin2017adversarial} to generate adversarial perturbations on 100 images from the KonIQ dataset using the KonCept NR model, testing multiple perturbation norms: \(\{0.02, 0.05, 0.1, 0.15, 0.20, 0.25\}\). We then measured the performance degradation of both the original KonCept model and the FS-IQA defended KonCept model under these stronger attacks (see Figure \ref{fig:beyond_guarantees}). FS-IQA’s scores decreased, matching human perception, while the undefended KonCept’s scores grew by up to 60\%. This shows FS-IQA stays robust even beyond its certified guarantees.

\begin{figure}[htb]
\begin{center}
\centerline{\includegraphics[width=0.8\linewidth]{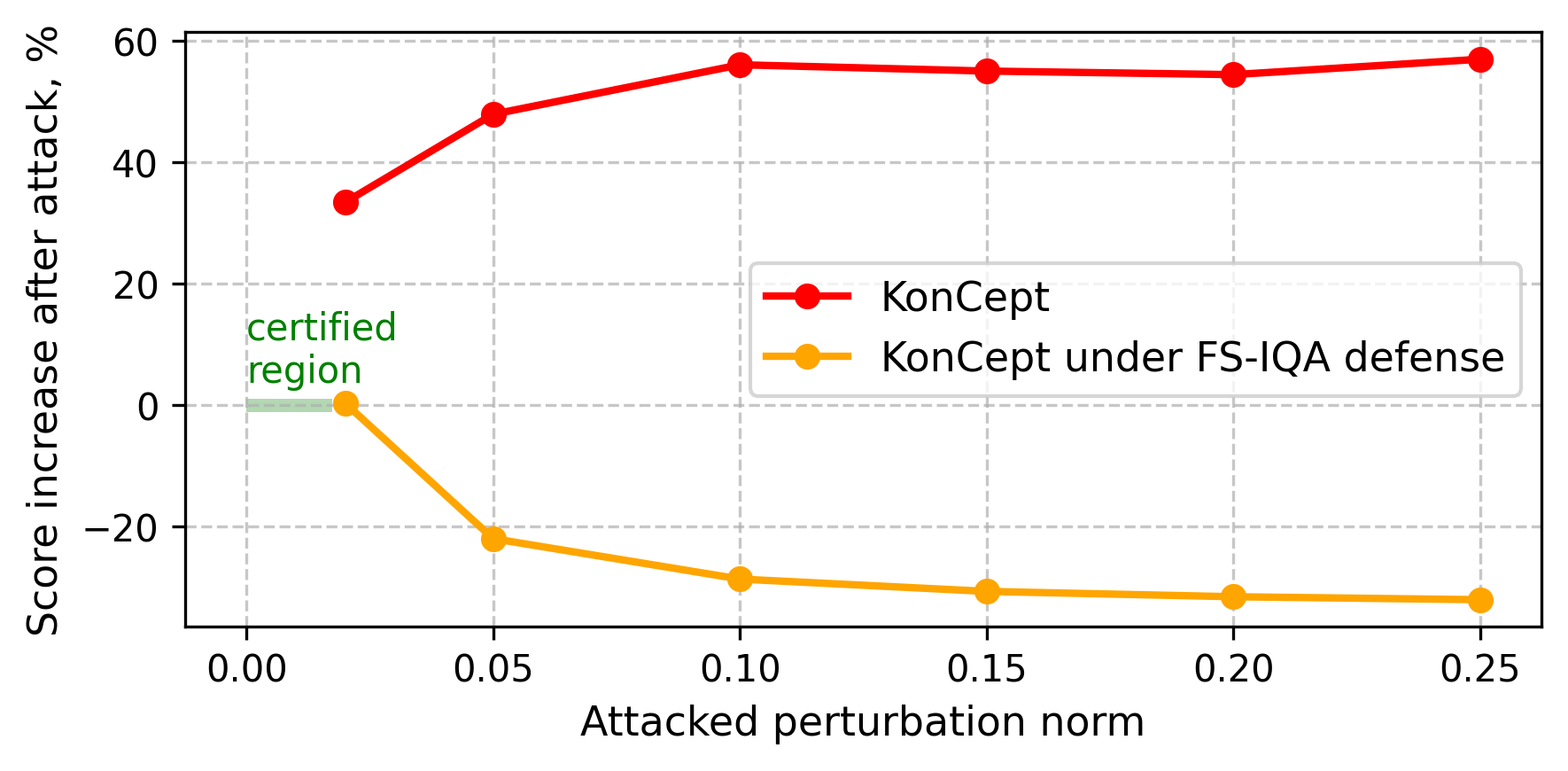}}
\caption{Evaluation of FS-IQA robustness beyond the certified region against adversarial perturbations generated by I-FGSM (10 iterations). }
\label{fig:beyond_guarantees}
\end{center}
\end{figure}

\subsection{Output stability}
FS-IQA is a stochastic method; however, we additionally demonstrated that running the model multiple times on the same image results in only minor variations in the output. We measured the deviation of the model’s output on a single image across multiple runs and found that the change in output is only 0.06\%. Details in the supplementary material.

\subsection{Limitations}
One main limitation of the FS-IQA approach is its high computational complexity in certification mode. Currently, the processing time is not suitable for real-time applications. However, with increased computational resources, this may become more practical. Therefore, we suggest using the FS-IQA method in real-time applications without the certification mode, combined with pre-testing on a specific data type using certification to estimate robustness, since similar data typically exhibit similar certified properties. As demonstrated in our experiments, this mode is fast—comparable to the speed of standard IQA models, while providing good empirical robustness. It is worth noting that previous certified methods without certification mode are still relatively slow, making our method more suitable for real-time use.

The second limitation is the number of abstentions. When the certification is not reliable, the method outputs an abstain decision. In all our experiments, the percentage of abstains was 1.5\%. Although this rate is not very high, it depends on the IQA backbone and the specific data. Therefore, in practice, users can select a suitable IQA backbone for the given data to minimize the number of abstains.

\section{Conclusion}
This paper introduces a novel certified defense method for IQA models based on randomized smoothing with noise applied in the feature space instead of the input space. Through extensive experiments on two benchmark datasets with six popular FR and NR IQA models, we showed that in contrast to prior approaches, our method achieves 30.9\% better correlation with subjective scores. Importantly, FS-IQA offers a significant speed advantage, especially in non-certification mode, boosting efficiency by 99.5\%, making it practical for real-time applications where certified defenses have historically struggled due to resource-intensive computations. Beyond theoretical guarantees, FS-IQA exhibited strong empirical robustness against perturbations significantly larger than the certified bounds, addressing common criticisms of certified defenses. Overall, FS-IQA presents a promising direction for robust and efficient certified defenses in IQA, balancing theoretical guarantees with practical usability. Code is publicly available at \textit{link is hidden for a blind review}.

\bibliography{aaai2026}


\end{document}